\documentclass[sigconf]{acmart}
\usepackage{booktabs} % For formal tables
\usepackage{amsmath}
\usepackage{graphicx}
\usepackage{xcolor,colortbl}
\usepackage{enumitem}
\usepackage{etoolbox}
\usepackage{hyperref}

\newtheorem{definition}{Definition}

\copyrightyear{2019}
 \acmYear{2019} 
\setcopyright{iw3c2w3g}
\acmConference[WWW '19 Companion]{Companion Proceedings of the 2019 World Wide Web Conference}{May 13--17, 2019}{San Francisco, CA, USA}
\acmBooktitle{Companion Proceedings of the 2019 World Wide Web Conference (WWW '19 Companion), May 13--17, 2019, San Francisco, CA, USA}
\acmPrice{}
\acmDOI{10.1145/3308560.3317593}
\acmISBN{978-1-4503-6675-5/19/05}

\begin{document}
\title{Nuanced Metrics for Measuring Unintended Bias \\ with Real Data for Text Classification}

%\author{Authors hidden for blind review}
% The ACM Guide said to do multiple authors sharing the
% same affiliation.
\author{Daniel Borkan}
\email{dborkan@google.com}
\affiliation{\institution{Jigsaw}}
\author{Lucas Dixon}
\email{ldixon@google.com }
\affiliation{\institution{Jigsaw}}
\author{Jeffrey Sorensen}
\email{sorenj@google.com }
\affiliation{\institution{Jigsaw}}
\author{Nithum Thain}
\email{nthain@google.com }
\affiliation{\institution{Jigsaw}}
\author{Lucy Vasserman}
\email{lucyvasserman@google.com }
\affiliation{\institution{Jigsaw}}

\renewcommand{\shortauthors}{Borkan et al.}

\makeatletter
\newcommand{\myitem}[2][]{%
  \ifblank{#1}{%
  \item \textbf{#2}%
  }{%
  \item[#1] \textbf{#2}%
  }%
  \protected@edef\@currentlabelname{(\theenumi) #2}
}
\makeatother

\begin{abstract}
Unintended bias in Machine Learning can manifest as systemic differences in performance for different demographic groups, potentially compounding existing challenges to fairness in society at large.
In this paper, we introduce a suite of threshold-agnostic metrics that provide a nuanced view of this unintended bias,
by considering the various ways that a classifier's score distribution can vary across designated groups.
We also introduce a large new test set of online comments with crowd-sourced annotations for identity references.
We use this to show how our metrics can be used to find new and potentially subtle unintended bias in existing public models.

\end{abstract}



\maketitle

\section{Introduction}

Evaluating the predictions of models is an essential part of machine learning work, and selecting the appropriate evaluation metrics for a given task is a central question for researchers in this field.
As machine learning models are being used for an ever-expanding set of use cases, many have raised concerns about the potential negative impact of unintended identity-based bias that could be present in the models \cite{compas, impossibility}.
In recent years, significant research has been published presenting new metrics intended to measure this unintended bias during model evaluation \cite{equalityofopp, disparate_impact, gender-shades}.

Our interest is in improving text classification models used to
identify toxicity in comments from online discussions, but the evaluation methods presented here can be applied to a broad range of classification applications.
``Toxicity'', defined as anything that is \textit{rude, disrespectful, or unreasonable that would make someone want to leave a conversation}, is an inherently complex and subjective classification task.
Machine learning systems, if not constrained, will often learn the simplest associations that
can predict the labels, so any incorrect associations present in the training data can produce unintended associations in the final model.
Toxicity models specifically have been shown to capture and reproduce biases common in society,
for example mis-associating the names of frequently attacked identity groups (such as ``gay'', and ``muslim'' etc.) with toxicity \cite{aies_2018,reducinggenderbias}.
This unintended model bias could be due to the demographic composition of
the online user pool, the latent or overt biases of those doing
the labelling, or the very selection and sampling process used to
choose which items to label.

Regardless of the source, we focus on measuring a particular aspect of
model unfairness - the skewing of classifier scores, and thus
output labels, due to identity related content within the text. 
We use a definition of model fairness similar to {\em equality of odds} defined in \cite{equalityofopp}.
As in that work, we assume the existence of a test set with reliable labels across a range of groups.
Given such a test set, we consider unintended bias to be present in the model if the model performance, according to relevant performance metrics, varies across the set of designated groups.
It is important to highlight that the assumption of reliable labels is significant and doesn't hold in all use cases.
We mitigate the impact of this assumption by demonstrating our results against both a synthetic test set with labels that are constructed to be reliable and a large human-annotated test set with high rating redundancy.

We propose a suite of threshold agnostic performance metrics to measure the extent of unintended model bias.
Many prior methods for measuring unintended bias in classification systems rely on selecting a threshold, a choice that can drastically change results.
In practice, classification models often return {\em scores} instead of binary classification decisions, enabling them to be used with a variety of thresholds or for ordering data.
For these models, threshold dependant metrics can obscure the view of unintended bias and thus be misleading to practitioners.
Threshold agnostic metrics capture the behavior of the underlying model
itself, and thus can allow
a more comprehensive comparison of the model's performance and
limitations. 

While the practice of building machine learning models is simplest when there are single metrics for model comparison,
unintended bias in models can be extremely nuanced and varied across groups and so single metrics are likely to obfuscate essential information.
We therefore propose a suite of five metrics, derived from ROC-AUC, Equality Gap, and Mann-Whitney U metrics,
each of which captures a different aspect of model performance, and  a different potential type of unintended bias.
Viewing a suite of metrics across a range of groups will provide much greater insight into the nuance of unintended bias, and hopefully reveal new opportunities for mitigating these biases. 

We apply these metrics with two test sets, again making the assumption that the labels are reliable.
One is a synthetic test set, identical to the one presented in \cite{aies_2018}.
The other, introduced in this work, is a new human-labeled dataset of nearly 2 million comments, specifically created for evaluation of unintended bias\footnote{https://git.io/fhpcC}.
This includes 450,000 comments annotated with the identities that are referenced in the text.

We demonstrate our proposed metrics and datasets on two publicly accessible models that are trained to detect toxicity in text (provided by the Perspective API~\cite{perspective_api}.)
One of these models is claimed to be trained using a bias mitigation technique, as described in \cite{aies_2018} and \cite{fp_blog}.
We show that our metrics and datasets illuminate unintended bias in the original model and provide new insight into the effectiveness of the bias mitigation between the models.

\section{Related Work}

Significant recent work has been published on defining how the concepts of fairness and unintended bias apply to machine learning models.
Much of this work provides metrics to quantify the presence of unintended bias according to specific definitions.
In \cite{impossibility}, \cite{tradeoffs}, and \cite{cost-of-fairness}, one can read overviews of several of these metrics and the relationship and trade-offs between these various metrics.  An industry
wide push towards
increased transparency about machine learning data sources, techniques,
and evaluation criteria, as advocated in \cite{modelcards}, has
underscored the importance of choosing the right metrics.

Recent works such as \cite{gender-shades} and \cite{compas} demonstrate the value of these metrics to highlight the potential for unfair impact with the use of machine learning models in
applications. 
In addition, several works, such as \cite{equalityofopp}, \cite{adversarial}, \cite{adversarialremoval} 
and \cite{aies_2018} provide techniques intending to mitigate
bias in machine learning models. 
While out of the scope of this work, effective unintended
bias measurement
is essential to measure progress in bias mitigation.

~\cite{reducinggenderbias} focuses on measuring and reducing 
the gender-based bias specifically.
In addition to using an approach from~\cite{aies_2018}, they
introduce measure of the false positive and negative equality
gaps, very similar to~\cite{adversarial}, that measure the
different between the true positive rates between the subgroup
and the overall background.

This is a relaxation of the {\em equalized odds} fairness constraint presented in \cite{equalityofopp} that can serve as a metric only
for a classifier that produces binary labels.
However, many models produce a probability distribution.
As these probability scores may be used in a variety of ways or with a variety of thresholds, the Equality Gap falls short for many purposes, as it evaluates the model only at one specific threshold.

The related work on counterfactual fairness \cite{counterfactual} proposes
a number of techniques to reduce classifiers affinity for assigning
attributed features to identity terms. As before, the equality
gap is measured for a particular threshold chosen to maximize the
accuracy over a designated test set.

\cite{aies_2018} introduces a threshold agnostic metric for unintended bias,
but a follow up work by the same authors, \cite{pinned-auc-limitations},
highlights several limitations of this metric.
Specifically, the metric is not robust to variations in the class distribution
between different identity groups.
In addition, with a single metric, some important information may be hidden as
different types of bias could obscure one another.

Our proposed metrics differ
from these early approaches because they are threshold agnostic, robust to class imbalances in the dataset, and 
because they provide more nuanced insight into the types of bias present in the model, as we will see in Section \ref{sec:comparison}.
\section{New Metrics}

\subsection{AUC-Based Metrics}

In this section we introduce three new metrics to measure unintended bias, based on the  Area Under the Receiver Operating Characteristic Curve (ROC-AUC, or AUC) metric.
For any classifier, AUC measures the probability that a randomly chosen negative example will receive a lower score than a randomly chosen positive example, i.e. that the two will be correctly ordered.
An AUC of 1.0 means that all negative/positive pairs are correctly ordered, with all negative items receiving lower scores than all positive items.

A core benefit of AUC is that it is threshold agnostic. 
An AUC of 1.0 also means that it is possible to select a threshold that perfectly distinguishes between negative and positive examples, i.e. that the classes are perfectly separable via the model score.

Most metrics for unintended bias rely on dividing the test data up by identity or demographic based subgroups and computing metrics for each group.
For our metrics, we also divide data by subgroup.
However, instead of calculating metrics on the subgroup data exclusively, our metrics compare the subgroup to the rest of the data, which we call the ``background'' data.

Consider the hypothetical score distributions for an example model shown in Figure~\ref{fig:distributions}.
The background score distributions are shown on the top, with negative examples in green on the left and positive examples in pink on the right.
The second (bottom) distributions shows scores for examples within a specific identity subgroup.
We can see clearly that the examples within the identity receive higher scores, both for positive and negative examples.
This score shift is one way that unintended bias can manifest in a model.
In fact, many types of unintended bias can be uncovered by looking at differences in the score distributions between background data and data from within a specific identity (although not all differences imply harmful bias). 
Using three new metrics based on AUC, we can specifically measure variations in the distributions that cause mis-orderings between negative and positive examples, i.e. variations that limit the possibility of selecting a single effective threshold. 

\begin{figure}[htb]
\includegraphics[scale=0.5]{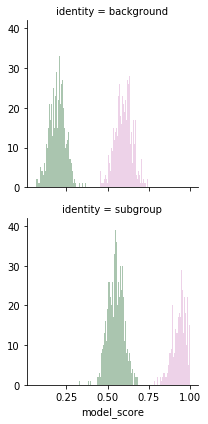}
\caption{Hypothetical model score distributions for the background data (top) and one identity subgroup (bottom).}
\label{fig:distributions}
\end{figure}

In the example shown in Figure~\ref{fig:distributions}, the dataset is divided into background and identity subgroups and negative and positive classifications, creating four distinct subsets: negative examples in the background, positive examples in the background, negative examples in the subgroup, and positive examples in the subgroup.
We define three AUCs to measure negative/positive mis-ordering between these four subsets. 

\begin{definition} Let $D^-$ be the negative examples in the background set, $D^+$ be the positive examples in the background set, $D_g^-$ be the negative examples in the identity subgroup, and $D_g^+$ be the positive examples in the identity subgroup.
\begin{eqnarray}
\textup{Subgroup AUC} & = & \textup{AUC }(D_g^- + D_g^+) \label{eq:subgroup_auc} \\
\textup{BPSN AUC} & = & \textup{AUC }(D^+ + D_g^-) \label{eq:bpsn_auc} \\
\textup{BNSP AUC} & = & \textup{AUC }(D^- + D_g^+) \label{eq:bnsp_auc}
\end{eqnarray}
\end{definition}

Table \ref{tab:metrics_examples} illustrates how each of these three AUCs is constructed by looking at a different subset of the data. Each of these three AUCs captures a unique and specific aspect of the model performance:
\begin{description}[style=unboxed]

\item[Subgroup AUC] Equation term~\ref{eq:subgroup_auc}, the calculates AUC on only the examples from the subgroup.
This represents model understanding and separability within the subgroup itself.

\item[Background Positive Subgroup Negative (BPSN) AUC] Equation term~\ref{eq:bpsn_auc} calculates AUC on the positive examples from the background and the negative examples from the subgroup.
This value would be reduced when scores for negative examples in the subgroup are \textit{higher} than scores for other positive examples, as in the example in Figure~\ref{fig:distributions}.
These examples would likely appear as \textit{false positives} within the subgroup at many thresholds.

\item[Background Negative Subgroup Positive (BNSP) AUC] Equation term~\ref{eq:bnsp_auc} calculates AUC on the negative examples from the background and the positive examples from the subgroup. 
This value would be reduced when scores for positive examples in the subgroup are \textit{lower} than scores for other negative examples. 
The examples would likely appear as \textit{false negatives} within the subgroup at many thresholds.

\end{description}

\begin{table*}[htb]
    \centering
    \begin{tabular}{|c|c|c||c|c|}
    \hline
    \multicolumn{5}{|c|}{\textbf{Unintended Bias Metrics}} \\
    \hline
    \multicolumn{3}{|c||}{\textbf{AUCs}} & \multicolumn{2}{c|}{\textbf{Average Equality Gaps}} \\
    \hline
    Subgroup & Background Positive & Background Negative  & Negative & Positive\\
    AUC & Subgroup Negative & Subgroup Positive & AEG & AEG \\
    & (BPSN) AUC & (BNSP) AUC & & \\
    \hline
    $\mbox{AUC}(D_g^- + D_g^+)$ &
    $\mbox{AUC}(D^+ + D_g^-)$ &
    $\mbox{AUC}(D^- + D_g^+)$ & 
    $\frac{1}{2} - \frac{\mbox{MWU}(D_g^-, D^-)}{|D_g^-||D^-|}$	&
    $\frac{1}{2} - \frac{\mbox{MWU}(D_g^+, D^+)}{|D_g^+||D^+|}$	\\
    \hline
    \includegraphics[scale=0.25,trim={0 10.3cm 26cm 0},clip]{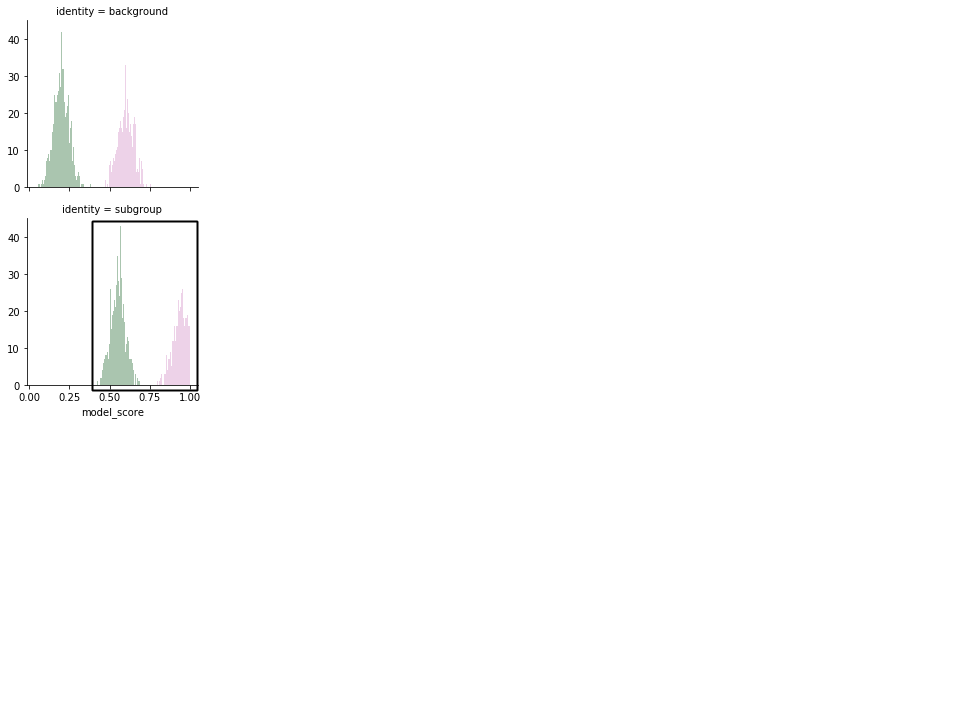} & 
    \includegraphics[scale=0.25,trim={0 10.3cm 26cm 0},clip]{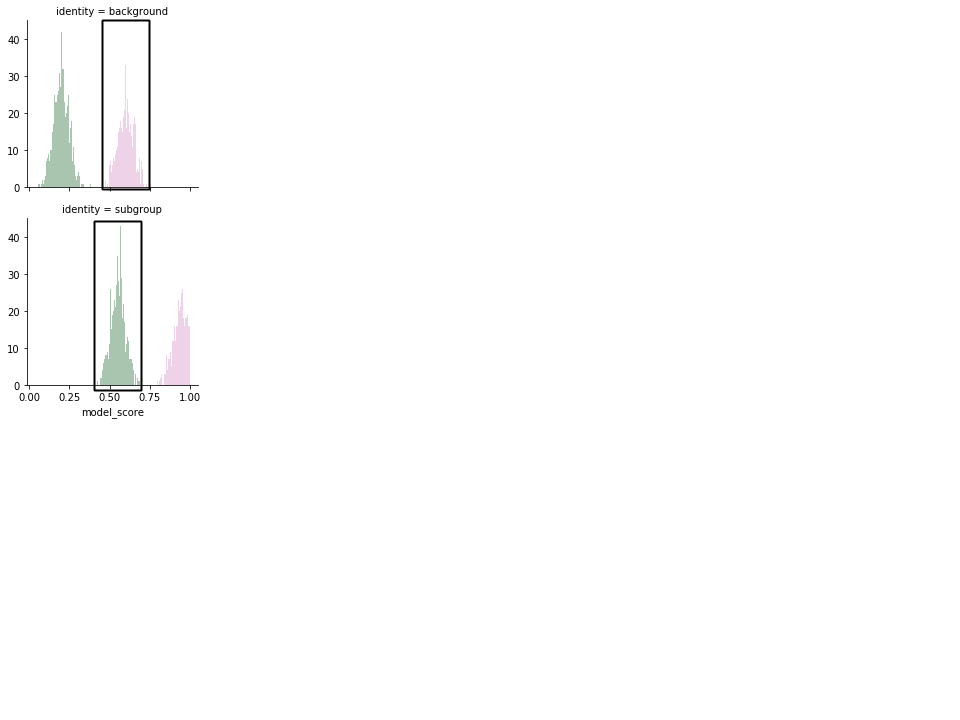} & 
    \includegraphics[scale=0.25,trim={0 10.3cm 26cm 0},clip]{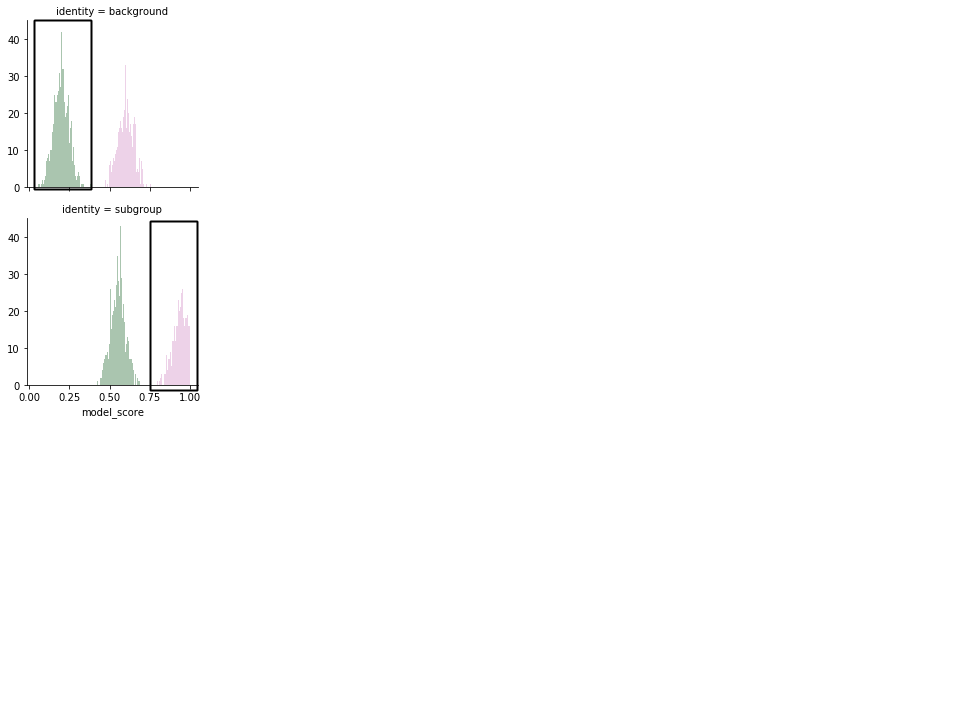} & 
    \includegraphics[scale=0.25,trim={0 10.3cm 26cm 0},clip]{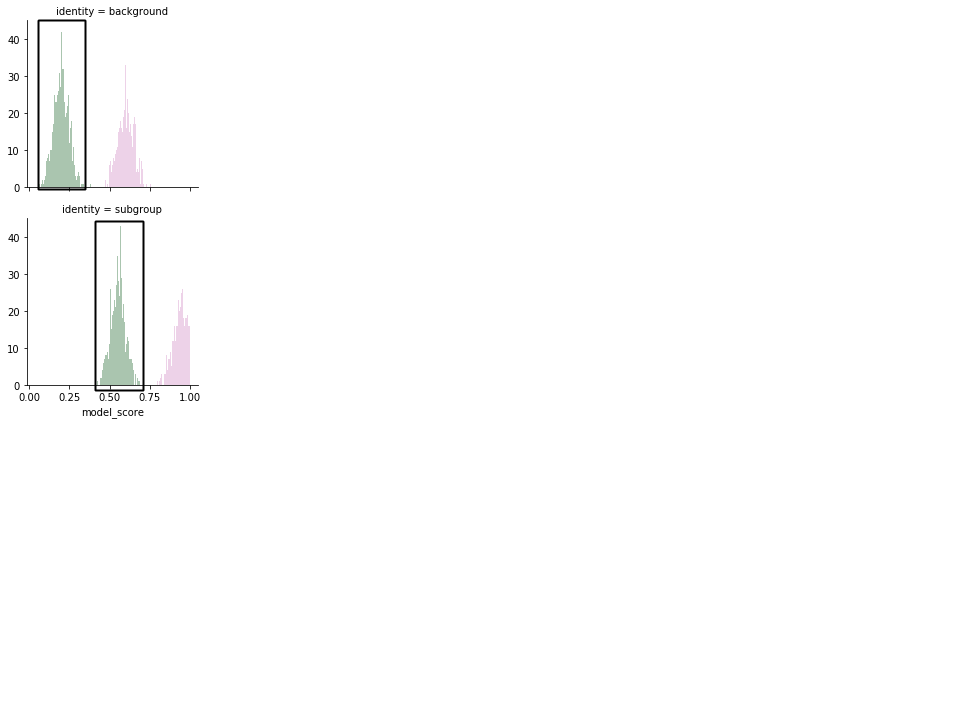} & 
    \includegraphics[scale=0.25,trim={0 10.3cm 26cm 0},clip]{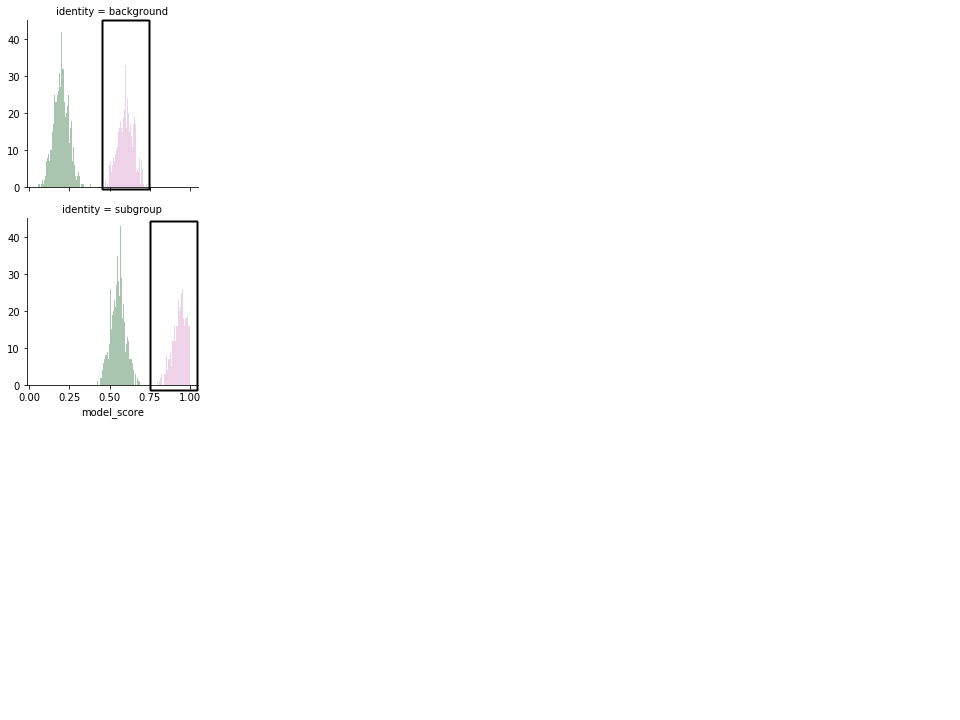} \\
    \hline
    \end{tabular}
    \caption{The full suite of unintended bias metrics. Highlighted distributions indicate which data is used to calculate each metric.}
    \label{tab:metrics_examples}
\end{table*}

Looking at these three metrics together for any identity subgroup will reveal how the model fails to correctly order examples in the test data, and whether these mis-orderings are likely to result in false positives or false negatives when a threshold is selected.

An important quality of the AUC metric is that it is robust to data imbalances in the amount of negative and positive examples in the test set.
This is especially relevant when measuring unintended bias, because in real-world data, the amount of examples in each identity subgroup, and the balance between negative and positive examples can vary widely across groups (in fact, this variation is often a source of bias).
Enforcing that for each AUC, either all negative or all positive examples (or both in Subgroup AUC) come from one identity group, means that mis-orderings involving that particular subset cannot be drowned out by results from other groups, ensuring that these metrics are robust to data imbalances likely to occur in real data. 

\subsection{Average Equality Gap}
We now introduce two additional threshold agnostic metrics, building from a strict generalization of the Equality Gap metric.

\begin{figure}[h]
\includegraphics[scale=0.2]{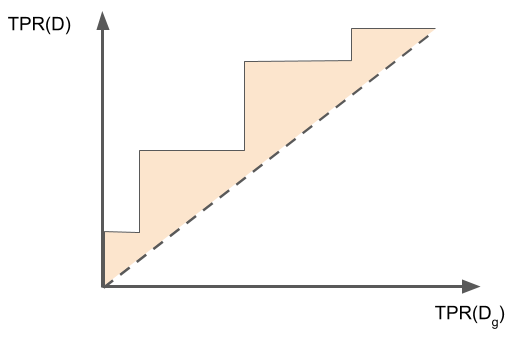}
\caption{An illustrative plot of the true positive rate of the subgroup and background distributions of a hypothetical classifier.
This shaded area can be captured by the AEG.
}
\label{fig:tpr}
\end{figure}

The Equality Gap is the difference between the true positive rates of the subgroup ($\mbox{TPR}(D_g)$) and the background ($\mbox{TPR}(D)$), at a specific threshold. 
Consider \ref{fig:tpr}, which plots these rates against each other for every possible threshold $t$, for some hypothetical classification model.
The hypothetical classifier is biased against the subgroup as $\mbox{TPR}(D_g) < \mbox{TPR}(D)$ at different levels at different thresholds. The shaded area captures the average bias across all thresholds for the classifier. 
\begin{definition}[\mbox{Positive Average Equality Gap}]
For each threshold, $t$, if you plot the true positive rate of the subgroup as $x(t)$ and the true positive rate of the background as $y(t)$ then the Positive Average Equality Gap is the area between the curve $\left(x(t), y(t)\right)$ and the line $y = x$, i.e.
\begin{equation}
\textup{Positive AEG} = \int_0^1 \left( y(t) - x(t) \right) \mathit{dx(t)} 
\end{equation}
\label{def:aeg_1}
\end{definition}
There is also the analogous definition with true negative rates in place of true positive ones. This would correspond to the other component of equalized odds. We call this the \emph{Negative Average Equality Gap}. Together, we call the pair the Average Equality Gap metrics (AEGs).

\subsubsection{Relationship to Mann-Whitney $U$ Metrics}

Another way to generalize the Equality Gap metric is from the perspective of the separability of the score distributions,
similar to the AUC metrics in the previous section.
With the AUC metrics, we measured mis-orderings between positive and negative examples across the subgroup and background, with the goal of few mis-orderings or high separability.
For the AEGs, we compare positive examples from the subgroup with positive examples from the background, with the goal of \textit{low} separability. 
In particular, if Equality of Opportunity held for our classifier at every threshold then that would mean that, if a point $i$ (with model score $\hat{Y}_i$) were chosen uniformly at random from the positive background data $D^+$ and a point $j$ (with model score $\hat{Y_j}$) were chosen uniformly at random from our subgroup data $D_g^+$, then: 
\[
P\left\{\hat{Y_i} > \hat{Y_j} | Y_i \in D^+, Y_j \in D_g^+\right\} = \frac{1}{2}
\]
That is to say, given that both data points are positive examples from the two distributions, the probability that either score is higher than the other should be the same.
Again, as a metric, we express this as:

\begin{definition}[Positive Average Equality Gap]
If a point $i$ (with model score $\hat{Y_i}$) were chosen uniformly at random from the background data $D^+$ and a point $j$ (with model score $\hat{Y_j}$) were chosen uniformly at random from our subgroup data $D_g^+$, then the average equality gap is:

\[
\textup{Positive AEG} = \frac{1}{2} - P\left\{\hat{Y_i} > \hat{Y_j} | Y_i \in D^+, Y_j \in D_g^+\right\}
\]
\label{def:aeg_2}
\end{definition}

We can rewrite the above definitions using the Mann-Whitney $U$ test statistic \cite{auc} to come up with an efficient closed form for computing the Positive Average Equality Gap: 

\begin{definition}[\mbox{Positive Average Equality Gap}]
\[
\textup{Positive AEG} = \frac{1}{2} - \frac{\textup{MWU}(D_g^+, D^+)}{|D_g^+||D^+|}	
\]
\label{def:aeg_3}
\end{definition}

Analogously, the equivalent definitions of Negative Average Equality Gap follows from simply substituting the negative datasets for the positive ones in the above definitions. 

Finally, we show:

\begin{theorem}
All definitions of Average Equality Gap are equivalent.
\end{theorem}

\begin{proof}
The equivalence of definitions \ref{def:aeg_2} and \ref{def:aeg_3} follow from the definition of the Mann-Whitney $U$ test statistic. The equivalence of \ref{def:aeg_3} and \ref{def:aeg_1} is shown identically to earlier proofs of the equivalence of AUC and MWU scores (see, for example Section 2 of \cite{auc}).
\end{proof}

\subsubsection{Properties of Average Equality Gap}

By Definition \ref{def:aeg_1}, we see that the Average Equality Gap as a metric can range in value from -0.5 to 0.5. 
At each of these extremes, it represents a different type of bias where the TPR of the subgroup is consistently higher or lower, respectively, than that of the background.

The optimal value of the Average Equality Gap metric is 0, which means the subgroup and background distributions have identical means.

Definition \ref{def:aeg_1} leads to the following corollaries:

\begin{theorem} \label{th:mp1} If Equality of Opportunity holds for every threshold then the Average Equality Gap will be 0.
\end{theorem}

\begin{proof} If Equality of Opportunity holds for every threshold then $x(t) = y(t)$ for all $t$ in Definition \ref{def:aeg_1}. Thus the Average Equality Gap will be 0. 
\end{proof}

\begin{theorem} \label{th:mp2} If the Average Equality Gap is 0 then Equality of Opportunity must hold for some non-trivial threshold 0 < t < 1.
\end{theorem}

\begin{proof} Equality of Opportunity always trivially holds if you threshold the classifier at 0 or 1. However, if the Average Equality Gap is 0 then, by the Average Value Theorem for Integrals, we must also have that $x(t) = y(t)$ for some $0 < t < 1$ where $x(t), y(t)$ are defined as in Definition \ref{def:aeg_1}.
\end{proof}
\subsection{Comparison of Metrics} \label{sec:comparison}

In this section, we discuss the different strengths and weaknesses of the AUC-based metrics and the AEGs at tackling a range of common biases. Table \ref{table:dist2} outlines simulated data distributions that exhibit common biases and reports our newly introduced metrics for bias. \\

\begin{enumerate}[label=\textbf{\Alph*},leftmargin=0.5cm]
\myitem{Small score shift}
This type of bias occurs when a machine learning classifier outputs a consistently higher (or lower) score for a subgroup than for the overall data distribution.
We refer to this as a ``small'' score shift as the shift is not to the extent as to confuse negative examples from the subgroup with positive examples from the background distribution (or vice versa). 
Thus, it is still possible to choose a single threshold that achieves a perfect separation of positive and negative examples for both the subgroup and the background. As we can see from Table \ref{table:dist2}, the Average Equality Gap (AEG) is the only metrics which pick up this subtle form of bias.
\label{p:small_shift}\\

\myitem{Large score shift} This type of bias is similar to the previous one, but the shift is now large enough that negative examples from the subgroup are mis-ordered with positive examples from the overall distribution. 
Thus, it is no longer possible for a single threshold to separate positive and negative examples from the subgroup or the background. 
A threshold that is ideal for the background distribution would result in false positives within the subgroup, captured by the low value in BPSN AUC. 
This score shift is also detected by both AEG metrics, which are higher here than in \ref{p:small_shift} capturing the greater extent of the bias in this scenario. \label{p:large_shift} \\

\myitem{Score shift and size skew (more positive)} The bias shift in this distribution is the same as the previous one, but the available subgroup data is skewed toward having many more positive than negative examples.
Again, this picked up by the BPSN and BNSP AUCs and AEGs, exactly as in \ref{p:large_shift}, demonstrating the metrics' robustness to imbalanced datasets. \label{p:large_shift_skew}\\

\myitem{Left score shift} This type of bias is similar to \ref{p:large_shift}, except that the subgroup scores are shifted downward (to the left). 
Here, the bias is captured again by negative values in both AEGs, indicating the downward shift in scores, and by the low value in BNSP AUC, indicating the likelihood of false negatives for the subgroup at thresholds ideal for the background. \label{p:left_shift}\\

\myitem{Low subgroup separability} This type of bias represents a classifier which simply underperforms on a subgroup relative to the background distribution, resulting in low separability within the subgroup only.
The intermingling of positive and negative examples from the subgroup is captured in the Subgroup AUC metric. 
Because this also implies a shift in one or both of the subgroup distributions, it will also be captured by the AEG  metrics. 
The sign of the AEG metric corresponds to the direction of the shift, so in this example, the -0.48 corresponds to the left-shift of the positive examples in the subgroup and the 0.48 corresponds to the right shift of the negative examples in the subgroup.
Note that the shifts are not so large as to cause overlap between negative examples in the subgroup and positive examples in the background (or vice versa), so the BPSN and BNSP AUCs are both still high. \label{p:lsubgroup} \\

\myitem{Wide subgroup score range without overlap} \label{p:wide_subgroup} This distribution represents the scenario where the classifier produces a higher variance of scores for the subgroup than the background distribution, but the means for these distributions are the same.
None of the bias metrics introduced here perceive any bias in this case.
Whether this is considered problematic bias depends on the use case, if it is, you'll want to use more sensitive metrics than those introduced in this work. \\

\myitem{Wide subgroup score range with overlap} \label{p:even_wider_subgroup} This distribution is similar to \ref{p:wide_subgroup}, except that the subgroup distributions are now so wide that they overlap with each other and with
the opposite class background distributions.
The AUC metrics, especially Subgroup AUC, perceive this type of bias. \\
\end{enumerate}

In summary, the different metrics can be used in combination to diagnose different types of bias. Of course, these scenarios are not an exhaustive list of all possible types of bias, but they help to illuminate some of the differences between these metrics.

Overall, Subgroup AUC and BPSN and BNSP AUCs identify any bias significant enough to cause mis-orderings between negative and postive examples, i.e. bias that interferes with selecting a single threshold that works similarly across groups.
Subgroup AUC highlights when those mis-orderings are caused by poor model understanding within the subgroup, and BPSN and BNSP AUCs highlight when the misorderings are caused by score shifts.
The AEGs go beyond the AUCs to identify bias in the distribution itself, even when (non-trivial) perfect thresholding is possible. 
Both AEGs and BPSN and BNSP AUCs provide insight into the directionality of score shifts. 

It's important to note that the correct handling of subtle variations in distributions must be decided on case by case basis.
Some cases of subtle variations that the metrics highlight, such as \ref{p:small_shift}, may not be considered problematic bias in all model use cases.
Other, even more subtle variations like \ref{p:wide_subgroup}, do not trigger any of these metrics.
Generally, if the AUC values reveal variation across groups, that likely indicates a problematic bias as it confirms mis-orderings around the relevant subgroups.
If the AEGs reveal variation across groups, that \emph{may} indicate problematic bias, however that bias may not be severe enough to cause mis-orderings.
And, as with any other suite of metrics, it's always possible that there is subtle bias that the metrics cannot detect.

\begin{table*}[htb]
\small
\begin{tabular}{ | l | c | c | c | c | c | c | c | c |}
\hline
& \ref{p:small_shift} &
\ref{p:large_shift} &
\ref{p:large_shift_skew} &
\ref{p:left_shift} & 
\ref{p:lsubgroup} & 
\ref{p:wide_subgroup} &
\ref{p:even_wider_subgroup} \\
\hline
Description & Small right & Large right & Score shift & Large left & Low subgroup & Wide subgroup & Wide subgroup \\
& score shift & score shift & and size skew  & score shift & separability & score range & score range \\
& & & (more positive) & & & without overlap & with overlap \\
\hline
Subgroup AUC & 1.0 & 1.0 & 1.0 & 1.0 & 0.93 & 1.0 & 0.92 \\
BPSN/BNSP AUCs & 0.99/1.0 & 0.76/1.0 & 0.76/1.0 & 1.0/0.77 & 0.99/0.99 & 1.0/1.0 & 0.98/0.97\\
AEGs & 0.42/0.42 & 0.5/0.5 & 0.5/0.5 & -0.5/-0.5 & -0.48/0.48 & -0.02/0.02 & 0.03/-0.03 \\
&
\includegraphics[scale=0.25]{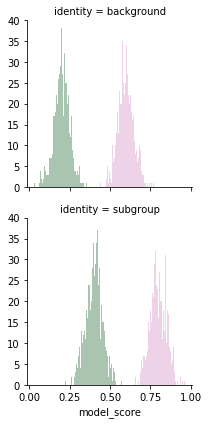} & 
\includegraphics[scale=0.25]{images/large_score_shift.png} & 
\includegraphics[scale=0.25]{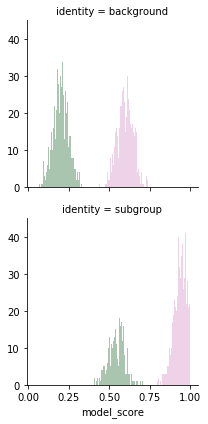} &
\includegraphics[scale=0.25]{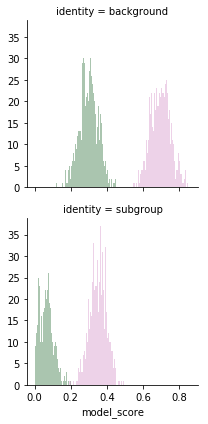} & 
\includegraphics[scale=0.25]{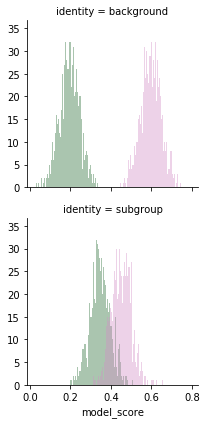} &
\includegraphics[scale=0.25]{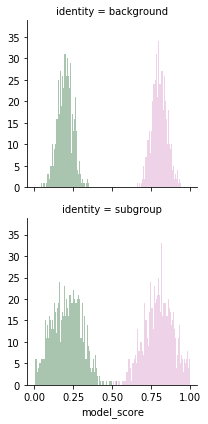} &
\includegraphics[scale=0.25]{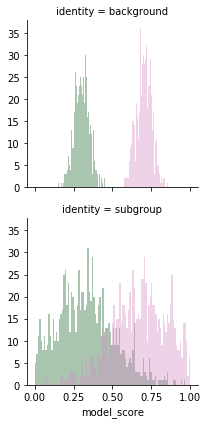} \\
\hline
\end{tabular}
\caption{Examples of score distributions demonstrating bias.}
\label{table:dist2}
\end{table*}

\section{Experimental Results}

We demonstrate this suite of metrics using the publicly available toxicity classifiers provided by the Perspective API (\cite{perspective_api}).
We use two test sets, 1) a synthetically generated, bias-focused test set following \cite{aies_2018} and 2) a large dataset of online comments with human labels for both identity and toxicity. 

\subsection{Models}
Using our metrics, we compare two versions of Perspective API's toxicity classifier, the initial TOXICITY@1 and the latest TOXICITY@6. 
TOXICITY@1 was shown to have significant unintended bias around identity words like "gay" and "transgender", both by independent analysis and by the Perspective team \cite{fp_blog}. 
TOXICITY@6 was built using the bias mitigation techniques presented in \cite{aies_2018} and, and therefore we expect to see reduced unintended bias between these two models across our new metrics.

\subsection{Synthetic Test Set} 

The synthetic dataset contains 77k examples generated from templates using 50 identity terms, 50\% toxic and 50\% non-toxic across all terms.
These examples are constructed explicitly to measure unintended bias based on identity terms. 
The examples are simple sentences that should be clearly toxic or clearly non-toxic, regardless of identity terms present.

\subsection{Synthetic Test Set Results}

In Table~\ref{tab:synthetic_data_results}, we show Subgroup AUC, BPSN AUC, BNSP AUC, Negative AEG, and Positive AEG for both TOXICITY models on the synthetic dataset.
The dataset contains 50 identity terms, here we show results for the lowest performing 20 subgroups.

\subsubsection*{Bias tends to skew towards toxicity for certain groups.} 
For TOXICITY@1, we see low values for BPSN AUC for identity terms \textit{homosexual}, \textit{gay}, and \textit{lesbian}, and to a lesser extent \textit{transgender} and \textit{heterosexual}.
This reveals the tendency of the TOXICITY@1 model to mis-associate those words with toxicity, and therefore produce high toxicity scores for non-toxic examples with these words, aligning with the findings of frequent false positives for these identity terms in \cite{fp_blog}.
This unintended bias is reduced in TOXICITY@6, but not completely eliminated.

Subgroup AUC and BNSP AUC show relatively high values across all groups, in both TOXICITY@1 and TOXICITY@6.
This emphasizes that the model is generally effective at distinguishing toxic from non-toxic examples within every group (subgroup AUC), even for the groups that show an incorrect tendency towards toxicity in the BPSN AUC discussed above.
In addition, the high BNSP values show that mis-orderings of toxic comments referring to identities that would lead to false negatives are rare.

\subsubsection*{Score distributions vary widely across groups.}
The Average Equality Gap values reveal a much more complex view of the model than the AUCs.
As expected, identities with low BPSN AUC (e.g. \textit{homosexual}, \textit{gay}, and \textit{lesbian}) also have high Negative AEG values, both indicating an upward shift (towards toxicity) of model scores for non-toxic items.
However, several identities, such as \textit{bisexual}, \textit{trans}, \textit{queer}, and \textit{black}, have high BPSN AUCs and also high Negative AEGs values.
For these groups, there is also an upward shift in model scores for non-toxic items, but it is not large enough to cause mis-orderings with toxic items. 
In addition, we see relatively large improvement in BPSN AUC values from TOXICITY@1 to TOXICITY@6, but only minimal change in the Negative AEGs for the same identities, confirming again that the upward shift of model scores for non-toxic comments with these identities is reduced but not eliminated.

\begin{table*}[htb]
    \centering
    \begin{tabular}{|c|c||c|c|}
    \hline
    \multicolumn{4}{|c|}{\textbf{Synthetic Test Set}} \\
    \hline
    \multicolumn{2}{|c||}{TOXICITY@1} & \multicolumn{2}{c|}{TOXICITY@6} \\
    \hline
    AUCs & AEGs & AUCs & AEGs \\
    \hline
    \includegraphics[scale=0.5]{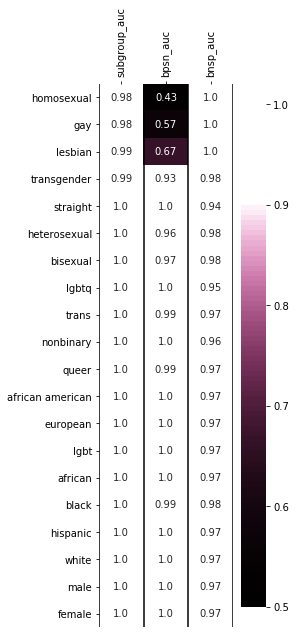} & 
    \includegraphics[scale=0.5]{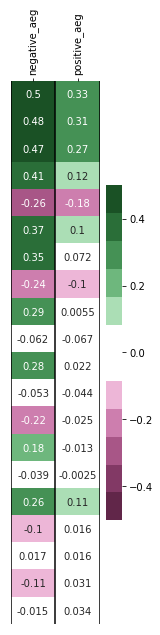} &
    \includegraphics[scale=0.5,trim=0 0 55 0, clip]{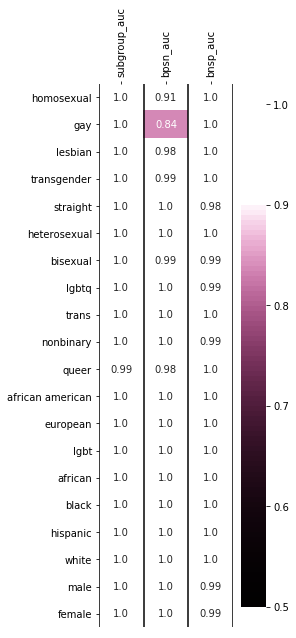} & 
    \includegraphics[scale=0.5,trim=0 0 55 0, clip]{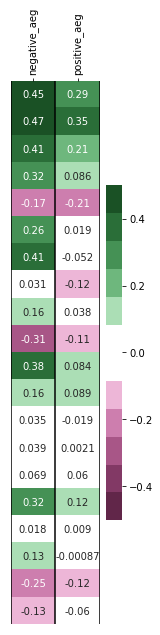}\\
    \hline
    \end{tabular}
    \caption{Comparison between TOXICITY@1 and TOXICITY@6 on the synthetic dataset.}
    \label{tab:synthetic_data_results}
\end{table*}

\subsection{Human Labeled Dataset}

\begin{table}[htb]
\small
\begin{tabular}{|p{2.8cm}|p{4.5cm}|}
\hline
\textbf{Category}  & \textbf{Identity Options}                                                                                            \\ \hline
Gender             & Male, Female, Transgender, Other gender                                                                              \\ \hline
Sexual Orientation & Heterosexual, Homosexual, Bisexual, Other sexual orientation                                                         \\ \hline
Religion           & Christian, Jewish, Muslim, Hindu, Buddhist, Atheist, Other religion                                                  \\ \hline
Race or ethnicity  & Black, White, Latino/Latina/Latinx, Other race or ethnicity                                                          \\ \hline
Disability         & Physical disability, Intellectual or learning disability, Psychiatric disability or mental illness, Other disability \\ \hline
\end{tabular}
\caption{Full list of identity options presented to crowd raters}
\label{identity_options}
\end{table}

Synthetic test sets, while useful for capturing issues not present in real data, may not provide accurate results for real scenarios with different data distributions.
In addition, synthetic sets are limited to the specific identity terms that are manually curated, and therefore are unlikely to capture the true diversity of ways that identities are discussed in real conversation.

To facilitate unintended bias evaluation on real data, we designed techniques to have humans label the identity content within real data.
We presented crowd raters with a comment and asked a set of questions including, for example, ``What genders are referenced in the comment?'' and ``What races or ethnicities are referenced in the comment?''.
For each question, raters selected the set of identities present in the comment from a provided list. 
Using human labeling for identity content allows us to capture nuanced identity content that term-based analysis would miss, such as ``same-sex marriage'' or ``people who do not believe in any god''.

The set of identities labelled by raters is not comprehensive and does not provide universal coverage. 
This set was designed to balance the coverage of identities, crowd rater accuracy, and ensure that each labeled identity has enough examples in the final data set to provide meaningful results.
The full list of identities labeled can be found in Figure~\ref{identity_options}.

This data was also labeled for toxicity using the same crowd rating guidelines as published by the Perspective API (\cite{Thain2017}, \cite{wulczyn2017ex}).
This labeling asks raters to rate the toxicity of a comment, selecting from ``Very Toxic'', ``Toxic'', ``Hard to Say'', and ``Not Toxic''.
Raters were also asked about several subtypes of toxicity, although these labels were not used for the analysis in this work.

Using these rating techniques we created a dataset of 1.8 million comments, sourced from online comment forums, containing labels for toxicity and identity.
While all of the comments were labeled for toxicity, and a subset of 450,000 comments was labeled for identity.
Some comments labeled for identity were preselected using models built from previous iterations of identity labeling to ensure that crowd raters would see identity content frequently.

Table \ref{table:toxicity_percent} shows the toxicity percentage for a selection of identities, illustrating that there is an imbalance in toxicity between different identities, emphasizing the value of metrics that are robust to these data skews.

\begin{table}[htb]
\begin{tabular}{|l|c|c|}
\hline
\textbf{Subgroup} & \textbf{Count} & \textbf{Percent Toxic}   \\ \hline
all comments &  1,804,875 &   8.00\%        \\
male         &  44,484    &   15.03\%       \\
female       &  53,429    &   13.68\%       \\
transgender  &  2,499     &   21.29\%       \\
heterosexual &  1,291     &   22.77\%       \\
homosexual   &  10,997    &   28.38\%       \\
\hline
\end{tabular}
\caption{Percentage of toxic comments by identity in the human labeled dataset for a selection of identities.}
\label{table:toxicity_percent}
\end{table}

To enable further research in this field, this entire dataset and annotations will be released under a Creative Commons license at \url{https://git.io/fhpcC}.

\subsection{Human Labeled Dataset Results}

\begin{table*}[htb]
    \centering
    \begin{tabular}{|c|c||c|c|}
    \hline
    \multicolumn{4}{|c|}{\textbf{Human Labeled - Short Comments Only}} \\ 
    \hline
    \multicolumn{2}{|c||}{TOXICITY@1} & \multicolumn{2}{c|}{TOXICITY@6} \\
    \hline
    AUCs & AEGs & AUCs & AEGs \\
    \hline
    \includegraphics[scale=0.5,trim=0 0 0 0,clip]{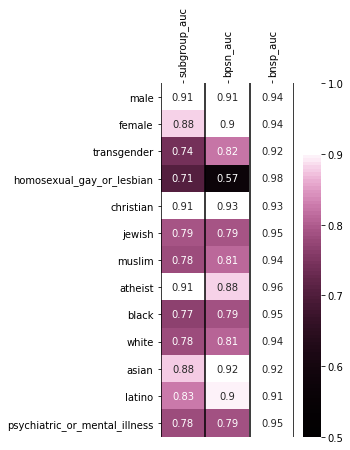} & 
    \includegraphics[scale=0.5,trim=0 0 0 0,clip]{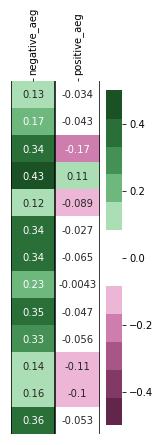} &
    \includegraphics[scale=0.5,trim=0 0 55 0,clip]{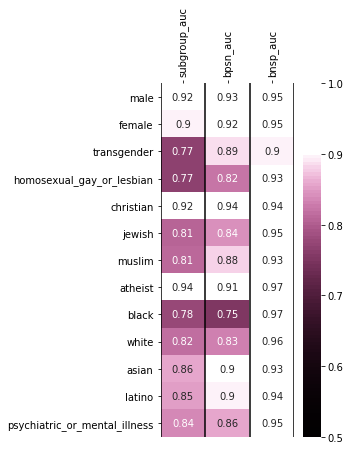} & 
    \includegraphics[scale=0.5,trim=0 0 55 0,clip]{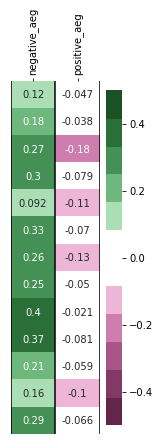} \\
    \hline
    \hline
    \multicolumn{4}{|c|}{\textbf{Human Labeled - All Comments}} \\  
    \hline
    \multicolumn{2}{|c||}{TOXICITY@1} & \multicolumn{2}{c|}{TOXICITY@6} \\
    \hline
    AUCs & AEGs & AUCs & AEGs \\
    \hline
    \includegraphics[scale=0.5,trim=0 0 0 0,clip]{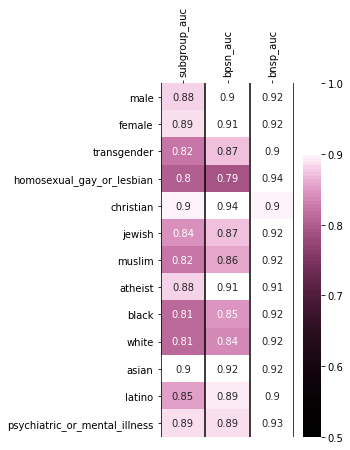} & 
    \includegraphics[scale=0.5,trim=0 0 0 0,clip]{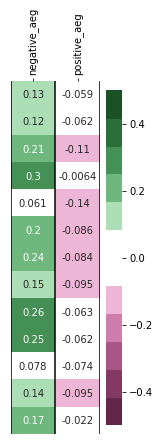} &
    \includegraphics[scale=0.5,trim=0 0 55 0,clip]{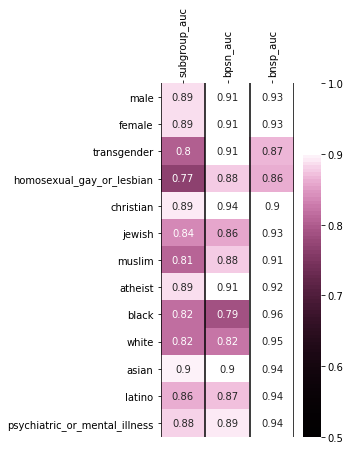} & 
    \includegraphics[scale=0.5,trim=0 0 55 0,clip]{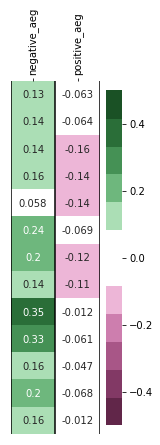} \\
    \hline
    \end{tabular}
    \caption{Comparison between TOXICITY@1 and TOXICITY@6 short comments and all comments in the human labeled dataset.}
    \label{tab:real_data_results}
\end{table*}

\begin{table}[htb]
    \centering
    \begin{tabular}{|c|c|}
    \hline
    \multicolumn{2}{|c|}{\textbf{Short Comments Only}} \\
    \hline
    \textbf{TOXICITY@1} & \textbf{TOXICITY@6} \\
    \hline
    \textbf{Background} & \textbf{Background}\\
    \includegraphics[scale=0.20]{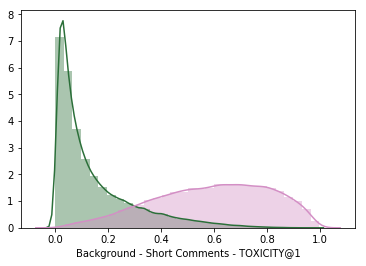} &
    \includegraphics[scale=0.20]{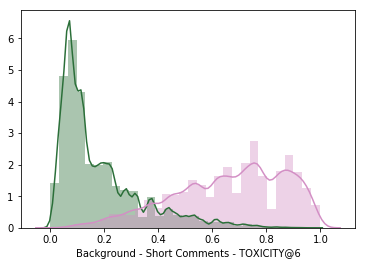}\\
    \textbf{homosexual\_gay\_or\_lesbian} & \textbf{homosexual\_gay\_or\_lesbian} \\
    \includegraphics[scale=0.20]{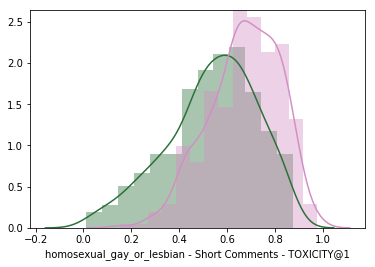} & 
    \includegraphics[scale=0.20]{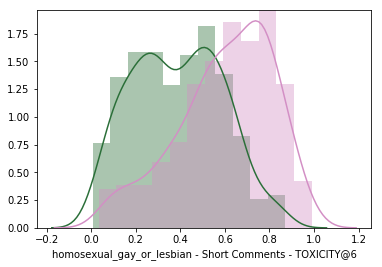}\\
    \hline
    \hline
    \multicolumn{2}{|c|}{\textbf{All Comments}} \\
    \hline
    \textbf{TOXICITY@1} & \textbf{TOXICITY@6} \\
    \hline
    \textbf{Background} & \textbf{Background}\\
    \includegraphics[scale=0.20]{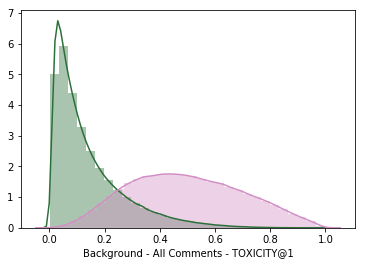} &
    \includegraphics[scale=0.20]{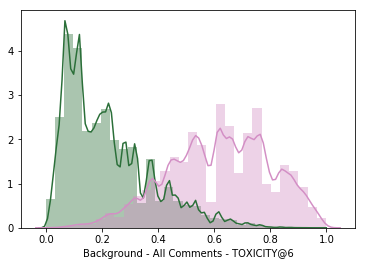} \\
    \textbf{homosexual\_gay\_or\_lesbian} & \textbf{homosexual\_gay\_or\_lesbian} \\
    \includegraphics[scale=0.20]{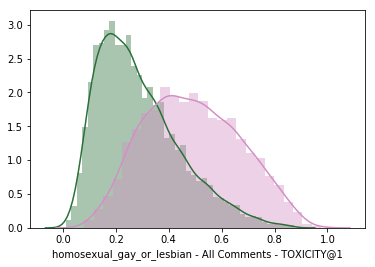} & 
    \includegraphics[scale=0.20]{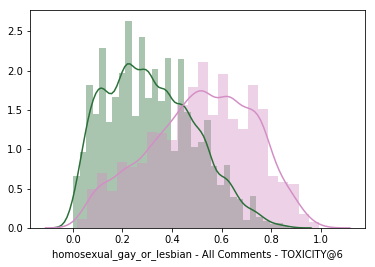} \\
    \hline
    \end{tabular}
    \caption{Score distributions comparison for comments labeled as containing identity \textit{homosexual} for short comments only and all comments, comparing TOXICITY@1 and TOXICITY@6.}
\label{table:bias_distribution}
\end{table}

Applying the AUC and AEG metrics to this real dataset reveals several new insights about the two toxicity models. 
Table \ref{tab:real_data_results} compares results for both TOXICITY@1 and TOXICITY@6 on all metrics, for both short comments (less than 100 characters) and all comments.

We present results on short comments separately because,
according to \cite{fp_blog},
the bias mitigation implemented between TOXICITY@1 and TOXICITY@6 focused on short comments. And, because, as we will see, results vary between short and long comments.
The identities shown in Table \ref{tab:real_data_results} are all identities that contained more than 100 examples of short comments. 
Walking through Table \ref{tab:real_data_results}, we highlight several important findings.

\subsubsection*{Real data reveals more unintended bias than synthetic data.}
Comparing the real data results to the synthetic data results in Table~\ref{tab:synthetic_data_results}, we find lower values and more variation across identity subgroups in the real data than we do in synthetic data.
However, the identity terms that show the lowest model performance in the synthetic data are aligned with the identity groups that show the lowest performance in the real data.
The synthetic data is intentionally very simple, so it is best at revealing very large discrepancies in performance that are tied very narrowly to specific identity terms, while the real data is much more broad and nuanced, but also potentially noisier.

\subsubsection*{Bias tends to skew towards toxicity.}
Across both models and both short and long comments, we see lower values for Subgroup AUC and BPSN AUC and higher values for BNSP AUC.
We also tend to see positive values for Negative AEG and negative values for Positive AEG.
Together, all of these metrics indicate that the models have a tendency to skew non-toxic comments that discuss identity towards toxicity.
Given that the domain is online discussions, the tendency towards toxicity for identity-related discussion is aligned with general societal perceptions of online conversation, and therefore may not be surprising.
In fact, we even see that the identities which have the most bias towards toxicity are identities that are frequently attacked in online discussion, such as \textit{black}, \textit{homosexual\_gay\_or\_lesbian}, and \textit{jewish} and the identities which have the least bias towards toxicity are \textit{male}, \textit{christian}, and \textit{asian}, again aligned with societal stereotypes. 
From these results and from what we know of society, we can conclude that for modeling toxicity in online conversation, the risk for bias skewed towards toxicity for some groups is high.
However, it's still important to include all metrics, as bias can appear in unexpected places.

\subsubsection*{Short comments show expected bias mitigation.} 
Looking at the AUCs and AEGs for short comments, TOXICITY@6 outperforms TOXICITY@1 in most identities and metrics.
This is expected since TOXICITY@6 includes data based on the bias mitigation methods described in \cite{aies_2018} and \citep{fp_blog}.
That work found that imbalances in toxicity in the training data for certain identity words were a major source of bias, and that those imbalances were more prevalent among short comments.
To mitigate the bias, the authors added additional training data to even out the prevalence of toxicity for those identity words, adding the most additional training data to short comments, where the imbalance was largest, so it is unsurprising that the mitigation would have more impact on short comments.
Finally, the identity groups which do not improve in these metrics, such as \emph{black} and \emph{asian}, were not associated with any additional non-toxic training data.

For all comments, we see much less change between TOXICITY@1 and TOXICITY@6, which raises the potential that the bias mitigation did not generalize to longer comments.
However, this dataset also reveals that there was \emph{more} unintended bias among short comments in TOXICITY@1, exactly as predicted by the training data analysis in \cite{aies_2018}.
This is evident by the fact that that for short comments the AUC values are lower (especially Subgroup AUC and BPSN AUC) and all metrics have more variation across groups.

Table \ref{table:bias_distribution} visualizes the larger bias for short comments and the improvement between TOXICITY@1 and TOXICITY@6 for the identity group \textit{homosexual\_gay\_or\_lesbian}. 
Each image shows the score distributions for non-toxic comments (green arc, towards the left) and toxic comments (purple arc, towards the right).
The top charts show the score distributions for both models on short comments and the lower charts shows all comments.
For short comments and TOXICITY@1, the non-toxic and toxic distributions for the identity group are almost entirely overlapping with each other, and with the background toxic distribution, illustrating the bias we see in the low Subgroup AUC, BPSN AUC, and Negative AEG. 
For short comments and TOXICITY@6, the non-toxic distribution for the identity group is shifted to the left (lower scores) showing some improvement, which aligns with the improvement we see in those metrics.

The metrics and distributions show that the bias mitigation brought unintended bias in short comments, where it was strongest, in line with unintended bias overall, where it was less visible.
These results mirror the bias-mitigating training data adjustment, where toxicity percentages among short comments with certain identity terms were brought in line with percentages for other words.
Overall, this evaluation on real data reveals that the bias mitigation between TOXICITY@1 and TOXICITY@6 impacts the model as designed, but there is still room for improvement.

\section{Conclusion and Further Work}

We introduced a new suite of metrics for unintended bias, based on ROC-AUC and Mann-Whitney $U$ scores. 
These metrics provide a detailed and nuanced view of the types of bias present in a model and overcome limitations of similar metrics like Equality Gap in that they are threshold agnostic.

We developed and applied an evaluation method for our introduced metrics using a variety of example illustrative distributions. This highlights the differences in various metric behaviors for different kinds of bias. 
We then demonstrated our metrics using existing toxicity classifiers that are provided by the Perspective API~\cite{perspective_api}.
This involved adapting existing synthetic datasets used for unintended bias measurement of text classifiers. 

Finally we extend beyond the synthetic test set methodology, leveraging the improved nuance of the newly introduced metrics by crowdsourcing a large new corpus of nearly 2 million annotations of comments, providing one of the first studies of unintended bias based on identity references in text classification on real data. 
Our evaluation using this new dataset highlights how the new metrics also reveal new challenges for bias mitigation, highlighting that bias is still present in models that have undergone some bias mitigation.

Further work in this area could study:

\begin{itemize}
\item Developing effective strategies for choosing optimal thresholds to minimize unintended bias. While the threshold agnostic metrics we present provide an understanding of bias in the underlying model scores, this does not mean that all thresholds will have the same results.
\item Evaluating the relative benefit of the newly introduced dataset compared to sub-string matching of terms that reference an identity. 
\item A more systematic definition of the kinds of synthetic distributions that can be used to evaluate and categorize metrics for unintended bias. 
\item Developing a full taxonomy of different possible biases and a systematic approach for these metrics to be used in their diagnosis. 
\end{itemize}

\bibliographystyle{ACM-Reference-Format}
\bibliography{main}
\end{document}